\documentclass{article}

 \usepackage[main, final]{neurips_2026}

\usepackage[utf8]{inputenc} 
\usepackage[T1]{fontenc}    
\usepackage{hyperref}       
\usepackage{url}            
\usepackage{booktabs}       
\usepackage{amsfonts}       
\usepackage{nicefrac}       
\usepackage{microtype}      
\usepackage{xcolor}         
\usepackage{mathtools}
\usepackage{subcaption}
\usepackage{siunitx}

\usepackage{colortbl}
\usepackage{enumitem}

\usepackage{tikz}
\usetikzlibrary{arrows.meta, positioning}
\usetikzlibrary{calc,shapes}

\usepackage{amsmath}
\usepackage{amssymb}
\usepackage{mathtools}
\usepackage{amsthm}

\usepackage{algorithm}
\usepackage{algorithmic}

\usepackage{hyperref}
\hypersetup{hidelinks}
\usepackage{xurl}

\usepackage{wrapfig}
\usepackage[normalem]{ulem} 
\usepackage{cancel}
\usepackage{soul}

\theoremstyle{plain}
\newtheorem{theorem}{Theorem}[section]

\newtheorem{lemma}[theorem]{Lemma}
\newtheorem{corollary}[theorem]{Corollary}
\theoremstyle{definition}
\newtheorem{definition}[theorem]{Definition}

\theoremstyle{remark}

\usepackage[most]{tcolorbox}

\definecolor{boxframe}{RGB}{60,90,140}   
\definecolor{boxfill}{RGB}{243,246,250}

\newtcolorbox{resultbox}{
  enhanced,
  colback=absfill,               
  colframe=absframe,             
  boxrule=0.6pt,                 
  arc=6pt,                       
  left=6pt, right=6pt, top=4pt, bottom=0pt,   
  before skip=8pt, after skip=8pt,            
}

\definecolor{absframe}{RGB}{205,197,182}
\definecolor{absfill}{RGB}{246,243,236}

\newtcolorbox{abstractbox}{
  enhanced,
  colback=absfill,
  colframe=absframe,
  boxrule=0.5pt,
  arc=8pt,
  left=14pt, right=14pt, top=10pt, bottom=4pt,
  before skip=4pt,
  after skip=26pt,
  drop fuzzy shadow=black!22,
}

\definecolor{linkcol}{RGB}{191,87,0}

\title{Emergence of Physical Intelligence \\ via Controllable Information Production}

\author{%
  Tristan Shah \\
  Department of Computer Science\\
  Texas Tech University\\
  \texttt{trisshah@ttu.edu} \\
  \And
  Stas Tiomkin\thanks{Corresponding Author} \\
  Department of Computer Science \\
  Texas Tech University \\
  \texttt{stas.tiomkin@ttu.edu} \\
}

\begin{document}

\maketitle

\begin{abstractbox}
{\bfseries Abstract}\par\smallskip
\noindent
    Intrinsic Motivation (IM) aims to train agents without external rewards, enabling useful behavior to emerge from the agent's interaction with its environment alone. However, the dominant IM approaches rely on information-theoretic quantities with designer-chosen variables, introducing bias and lacking a principled connection to dynamics or optimal control (OC). We introduce Controllable Information Production (CIP), a novel IM method explicitly grounded in dynamical systems and OC. CIP measures the rate at which an agent produces information, capturing controllable complexity without external knowledge or bias. CIP unifies IM and OC into a single framework, formalizing physical intelligence as the control of information production. It further reveals connections between the structure of the value function and Kolmogorov-Sinai entropy. CIP consistently outperforms prior IM methods on standard benchmarks in robot learning and solves tasks they fail on, including humanoid self-righting. These results support an organizing principle: that physical intelligence emerges from driving systems toward the edge of controllable chaos.

\par\medskip
{\small\textbf{Project page \& code:}
\href{https://controllable-information-production.netlify.app}{\color{linkcol}\nolinkurl{controllable-information-production.netlify.app}}\par}
\end{abstractbox}

\begin{figure}[H]
\vspace{-0.7cm}
    \centering
    \includegraphics[width=\linewidth]{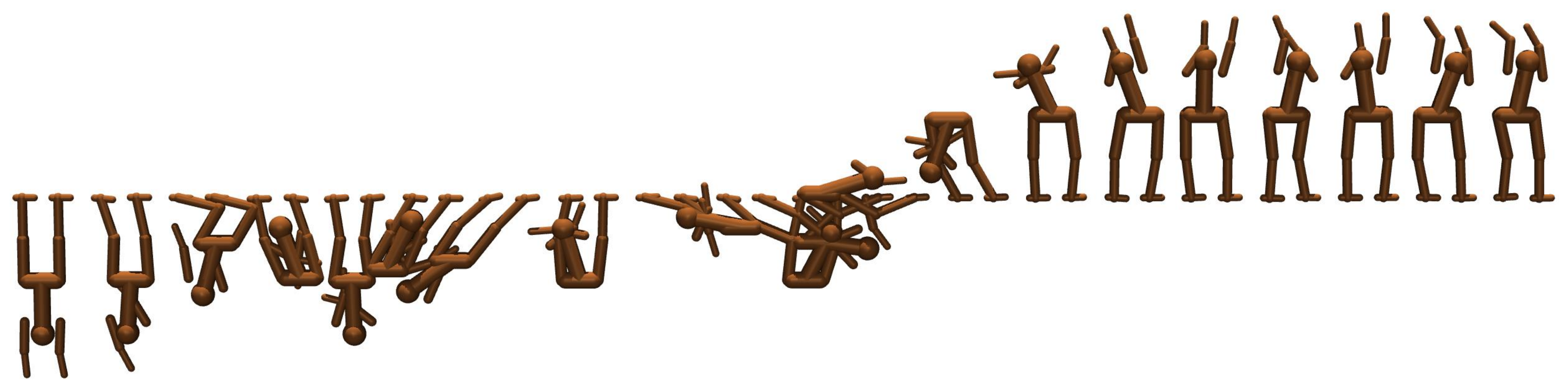}
    \caption{A humanoid-like agent discovers self-righting and stable upright behavior using Controllable Information Production (CIP) as its sole intrinsic motivation signal.}
    \label{fig:gibbon}
\end{figure}

\section{Introduction}\label{sec:introduction}


Robot learning has achieved substantial empirical success, yet its dependence on engineered reward signals and human operator demonstrations remains a fundamental bottleneck \citep{singh2019end,park2024automatic, mandlekar2023mimicgen, walia2025armimic}. These human interventions inject priors that limit the discovery of genuinely novel behaviors. More objectively, human expertise is not always available, and for some tasks humans may not know the solution themselves. Intrinsic Motivation (IM) offers an alternative paradigm in which intelligent behavior emerges from the agent's own interaction with a dynamical system, without externally specified goals. A range of information-theoretic objectives such as Empowerment \citep{klyubin2005empowerment}, curiosity \citep{schmidhuber2010formal}, information gain \citep{sukhija2025maxinforlboostingexplorationreinforcement}, diversity \citep{eysenbach2018diversityneedlearningskills, sharma2019dynamics}, and predictive information \citep{tishby2010information} have been proposed for IM and successfully applied to control problems including stabilization \citep{tiomkin2024intrinsic}, object manipulation \citep{ferraro2025focus}, locomotion \citep{karl2017unsupervisedrealtimecontrolvariational}, and transportation \citep{papala2024decentralized}. Despite these successes, several fundamental limitations remain.

Current IM methods suffer from three structural deficiencies. First, they are predominantly based on mutual information, which is at best a proxy for system-level phenomena in the underlying dynamical system such as complexity, stability, or chaos. Second, they require a designer to explicitly specify random variables over which to calculate mutual information. The resulting IM is biased by a designer-specified information interface which implicitly constrains how the agent perceives and acts. Third, although these objectives are routinely used to drive control, they have no principled connection to optimal control (OC) theory which limits interpretability and efficient methods based on OC principles. Ultimately, these objectives are postulated based on how human designers believe intelligent agents should behave rather than emerging from the very field that formalizes what it means for an agent to be in control \citep{tassa2012synthesis, todorov2005generalized, li2004iterative}.

We introduce a new IM method derived directly from invariant properties of dynamical systems. Rather than quantifying information transfer between hand selected random variables, we measure the rate at which the controlled system produces information along its trajectories. This rate is given by the sum of positive \textit{open-loop} Lyapunov Exponents (LE) \citep{oseledets1968multiplicative} induced by the agent’s actions, yielding the objective we call \textbf{Controllable Information Production (CIP)}.

CIP eliminates the above deficiencies. CIP is an intrinsic invariant of the dynamics, not a proxy defined over arbitrarily chosen random variables. Information production is defined at the level of the dynamical system itself, removing the need to select random variables altogether. Finally, CIP is not postulated: we show that it emerges directly from OC theory, where it is linked to the curvature of the optimal value landscape under an arbitrary control objective. This connection also yields an efficient and scalable estimator for Kolmogorov-Sinai entropy (KSE), which has historically been challenging to compute in high dimensions \citep{murphy2024machine, SHIOZAWA2024129531, baptista2010kolmogorov, wolf1985determining}.



Conceptually, CIP shifts the objective of IM from promoting novelty or diversity to identifying and stabilizing \textit{salient} regions of state space. These are regions that are intrinsically unstable or chaotic in open loop, yet controllable under feedback. Such regions confer high agency: for example, the upright configuration of a humanoid enables a wide range of behaviors while remaining highly sensitive to small perturbations. More broadly, intelligent behavior in nature is often associated with operating at the \textit{edge of chaos} \citep{tower2024selectively, wissner_2013}, where systems balance instability and control. By maximizing the information production an agent can control, CIP provides a principled mechanism for driving agents toward and maintaining operation within these regimes.


Our main contributions are as follows:\vspace{-0.3cm}
\begin{enumerate}[leftmargin=*, itemsep=2pt, parsep=0pt, topsep=4pt]
    \item We introduce a new intrinsic motivation (IM) method based on Controllable Information Production (CIP), driving the emergence of useful physical behaviors.
    \item We connect CIP to value function structure, grounding IM in optimal control (OC) and enabling an efficient, scalable estimator with a numerically stable optimization method.
    \item We establish state-of-the-art performance in intrinsic motivation: CIP outperforms prior methods and uniquely enables useful physical skills such as self-righting.
\end{enumerate}

\section{Preliminaries}\label{sec:kse}

Kolmogorov-Sinai entropy (KSE) quantifies the rate at which a dynamical system $g: \mathcal{X} \to \mathcal{X}$ produces new, distinguishable trajectories \citep{Kolmogorov1958,Sinai1959,Walters1982,Petersen1983}, and is strictly positive for chaotic dynamics. A finite measurable partition $\xi = \{P_1, \ldots, P_k\}$ of state space $\mathcal{X}$ acts as a coarse measurement: each step of the dynamics splits cells of $\xi$ into finer pieces, so longer horizons distinguish more trajectories. 
\begin{equation}\label{eq:kse-definition}
    h_{\mbox{ks}}(g) \coloneq \sup_{\xi} \lim_{T \to \infty} \frac{1}{T} H\!\left( \bigvee_{t=0}^{T-1} g^{-t} \xi \right).
\end{equation}
Here $H(\cdot)$ denotes the Shannon entropy of the partition $\xi$, $g^{-t}\xi$ its preimage under $t$ steps of the dynamics, and $\bigvee$ denotes the join of partitions. KSE is the supremum, over all finite partitions, of the per-step entropy growth rate of the joined partition.

The partition-based definition is rarely tractable due to the supremum over $\xi$ \citep{murphy2024machine}. It can instead be calculated by the sum of positive Lyapunov exponents (LEs) of the system, denoted as $\Lambda^+(g) \coloneqq \sum_{\lambda_i > 0}\lambda_i$, which are based on sensitivity analysis \citep{oseledets1968multiplicative,wolf1985determining}. Pesin's theorem \citep{pesin1977characteristic} formalizes this connection:
\begin{equation}\label{eq:sum_of_positive}
    h_{\mbox{ks}}(g) = \Lambda^+(g)
\end{equation}
where equality holds if $g$ is a $C^2$ diffeomorphism of a compact manifold and the invariant measure is ergodic with absolutely continuous conditional measures on unstable manifolds \citep{ledrappier1985metric}.

However, standard estimators of LE scale poorly to high-dimensional systems \citep{skokos2009lyapunov}, and evaluating Equation~\eqref{eq:sum_of_positive} requires resolving the full spectrum and discarding the negative exponents; no existing method tracks the positive sum directly. One of our contributions closes this gap.

In this work we are interested in the information production induced by the agent on the dynamical system. We now distinguish two quantities corresponding to information production under closed and open-loop control.

\paragraph{Open and Closed Loop Control.} We consider a dynamical control system $f: \mathcal{X} \times \mathcal{U} \to \mathcal{X}$ where $\mathcal{U}$ is the space of admissible controls. Under a feedback policy $u = \pi(x)$, the system reduces to: $ f^{\mathbf{cl}}(x) = f(x, \pi(x))$. Its sensitivity, required for the calculation of KSE, is governed by the total derivative of the closed-loop dynamics \citep{mayne_1966}:
\begin{equation*}
    \frac{d}{dx} f^{\mathbf{cl}}(x) = f_x + f_u \pi_x,
\end{equation*}
where $f_x \coloneq \frac{\partial f}{\partial x}$ and $f_u \coloneq \frac{\partial f}{\partial u}$ are the Jacobians of the dynamics with respect to state and control and $\pi_x = \frac{\partial \pi}{\partial x}$ is the Jacobian of the policy. Since $f^{\mathbf{cl}}$ is an autonomous map, its KSE is defined and, under the conditions of Equation~\eqref{eq:sum_of_positive}, equals $\Lambda^+(f^\mathbf{cl})$. It measures the information production rate that exists under feedback control and quantifies the residual instability the policy fails to suppress. An optimal controller drives this quantity down by canceling unstable perturbations as we prove in Theorem \ref{theorem:controllable_information_production}. 

In the open-loop system the control sequence is state-independent. Its sensitivity is governed by the open-loop Jacobian alone:
\begin{equation*}
    \frac{d}{dx}f^{\mathbf{ol}}(x) = f_x,
\end{equation*}
with no $\pi_x$ term, since the control sequence does not respond to perturbations in $x$.
Since the control sequence is state-independent, $f^\mathbf{ol}$ is not an autonomous map and its KSE is not defined by Equation~\eqref{eq:kse-definition}; its LEs along the nominal trajectory are well-defined (Lemma~\ref{lemma:kse_equivalence}), and $\Lambda^+(f^\mathbf{ol})$ measures the information production rate if perturbations are left unregulated. This is the quantity we propose as our intrinsic motivation signal which is formally stated in Definition \ref{def:cip}.

In the next section we show how CIP emerges from OC which is outlined by the roadmap in Figure~\ref{fig:proof_map}.

\section{Derivation of Controllable Information Production}\label{sec:theory}

{\noindent{\bf Overview.}} In this section, we show that CIP is encoded in the Hessian of the value function in OC, defined for arbitrary costs. This arises from decomposing the optimal policy into extrinsic (goal-directed) and intrinsic (perturbation-regulating) components, the latter governed by value curvature. By further decomposing the value Hessian into open- and closed-loop terms, we derive a Riccati equation characterizing controllable perturbation growth which calculates positive LE, yielding a tractable formulation of KSE. This connection enables a stable and efficient recursive estimator of CIP, validated on chaotic attractors with known KSE. Figure~\ref{fig:proof_map} outlines the derivation pipeline.
\begin{figure}[H]
\vspace{-0.25cm}
\centering
\begin{tikzpicture}[
    scale = 0.8,
    every node/.style={font=\footnotesize},
    box/.style={draw=absframe, fill=absfill, line width=0.6pt, rectangle, rounded corners=4pt, align=center},
    circ/.style={draw, circle, minimum size=1cm, align=center},
    arrow/.style={->, thick, black, line width=1.0pt},
    line/.style={thick, black},
    labelnode/.style={font=\footnotesize, align=center}
]
\node[box] (root) {Policy \\ Decomposition \\ Lem.~\ref{lemma:policy_decomp}};
\node[box] (intrinsic) at ($(root.east)+(2.0cm,0.75cm)$) {Intrinsic Part\\$\pi_{x_t}$ in Eq.~\eqref{eq:policy_decomp}};
\node[box] (extrinsic) at ($(root.east)+(2.0cm,-0.75cm)$) {Extrinsic Part\\$d_t$ in Eq.~\eqref{eq:policy_decomp}};
\node[box] (hessian) at ($(intrinsic.east)+(2.0cm,0)$) {Value Hessian \\ Decomposition \\ Lem.~\ref{lemma:value_decomp}};
\node[box] (ol) at ($(hessian.east)+(2.50cm,0.75cm)$) {Open-Loop\\Info. Prod., Eq.~\eqref{eq:open_loop_entropy}};
\node[box] (cl) at ($(hessian.east)+(2.50cm,-0.75cm)$) {Closed-Loop\\Info. Prod., Eq.~\eqref{eq:closed_loop_entropy}};
\node[box] (cip) at ($(ol.east)+(2.5cm,0)$) {Controllable \\ Information Production \\ Def.~\ref{def:cip}};
\draw[arrow] (root.east) -- (intrinsic.west);
\draw[arrow] (root.east) -- (extrinsic.west);
\draw[arrow] (intrinsic.east) -- (hessian.west);
\draw[arrow] (hessian.east) -- (ol.west);
\draw[arrow] (hessian.east) -- (cl.west);
\draw[arrow] (ol.east) -- (cip.west);
\end{tikzpicture}

\caption{Overview of the CIP derivation. The optimal policy is decomposed into intrinsic (feedback) and extrinsic (drift) components (Lemma~\ref{lemma:policy_decomp}). Analysis of the value Hessian then separates open- and closed-loop information production rates (Lemma~\ref{lemma:value_decomp}). The open-loop component defines CIP (Definition~\ref{def:cip}), which upper-bounds the closed-loop information production (Theorem~\ref{theorem:controllable_information_production}).}
\label{fig:proof_map}
\end{figure}
\vspace{-0.5cm}
\paragraph{Perturbation Regulation in Optimal Control.}
CIP is defined through the growth of infinitesimal perturbations. We analyze perturbations through local linearization along a nominal trajectory, which is not a restriction of the nonlinear dynamics but the standard approach for sensitivity analysis and exact computation of KSE \citep{oseledets1968multiplicative, sandri1996numerical, shaw1981strange, dieci1997computation}.

The system dynamics:
\begin{equation*}
x_{t+1} = f(x_t, u_t) \quad \text{and} \quad \delta x_{t+1} = f_{x_t}\delta x_t + f_{u_t}\delta u_t
\end{equation*}
describe the evolution of infinitesimal perturbations $\delta x_t$ around a nominal trajectory. A policy $\pi$ induces a value function, $V^\pi$, with an arbitrary running cost, $c_t(x, u)$, and terminal cost $c_T(x)$ \citep{todorov2005generalized, tassa2012synthesis, li2004iterative}:
\begin{equation*}
V^\pi(x_t) = \sum_{k=t}^{T-1} c_k\big(x_k, \pi(x_k)\big) + c_T(x_T), \qquad x_{k+1} = f(x_k, \pi(x_k))
\end{equation*}
with Hessian, $V_{xx_t}^\pi \coloneq \frac{\partial^2 V^\pi(x_t)}{\partial x_t^2}$, which captures the local curvature, i.e., the second-order sensitivity of future cost-to-go with respect to perturbations in the state. It quantifies how infinitesimal perturbations expand or contract under the dynamics and thus governs their propagation. In the following sections we take $\pi$ to be optimal with respect to $V^\pi$.

The key observation is that, for arbitrary costs the optimal value function implicitly encodes both goal-directed behavior and goal-independent perturbation regulation. In particular, its Hessian $V_{xx}^\pi$ governs second-order sensitivity to perturbations and determines how they propagate through the dynamics, independently of the specific task.

\subsection{Intrinsic–Extrinsic Policy Decomposition}

OC policies admit a natural decomposition into goal-directed and goal-agnostic components. We prove that the goal-directed component $d_t$ is a drift term, associated with the gradient of the value function $V_{x_{t+1}}^\pi$, that drives the system to minimize extrinsic cost. The goal-agnostic component $\pi_{x_t}$, associated with the Hessian of the value function $V_{xx_{t+1}}^\pi$, regulates and decays unstable perturbations. Remarkably, for quadratic costs, the Hessian depends only on the gradients of the local dynamics.

\begin{resultbox}
\begin{lemma}[Intrinsic-Extrinsic Policy Decomposition]\label{lemma:policy_decomp}

The optimal linear feedback policy is decomposed into a goal-directed extrinsic component and a goal-agnostic intrinsic component:
\begin{align}
    \pi(x_t) &= \bar{u}_t + \underbrace{d_t}_{\textrm{Extrinsic}} + \underbrace{\pi_{x_t}}_{\textrm{Intrinsic}}\delta x_t \label{eq:policy_decomp} \\
    \pi_{x_t} &= -(c_{uu_t} + f_{u_t}^\top V_{xx_{t+1}}^\pi f_{u_t})^{-1}f_{u_t}^\top V_{xx_{t+1}}^\pi f_{x_t} \label{eq:policy_gradient}
\end{align}
with the Hessian of the value function satisfying the Discrete Algebraic Riccati Equation (DARE):
\begin{equation}\label{eq:dare}
        V_{xx_t}^\pi = c_{xx_t} + f_{x_t}^\top V_{xx_{t+1}}^\pi f_{x_t} -f_{x_t}^\top V_{xx_{t+1}}^\pi f_{u_t} (c_{uu_t} \!+ \!f_{u_t}^\top V_{xx_{t+1}}^\pi f_{u_t})^{-1} f_{u_t}^\top V_{xx_{t+1}}^\pi f_{x_t}
\end{equation}
where $V_{xx_T}^\pi = c_{xx_T}$ is the terminal condition.
\end{lemma}
\end{resultbox}

\begin{proof}[Proof sketch]
An optimal linear feedback policy separates into a feedforward term involving the value gradient and a feedback term involving the value Hessian. The value gradient depends explicitly on the goal, making the feedforward term extrinsic; the value Hessian satisfies the goal-free Equation~\eqref{eq:dare}, making the feedback term intrinsic. See Appendix~\ref{app:policy_decomp}.
\end{proof}

In OC, the Hessians of the cost, $c_{xx_t}$ and $c_{uu_t}$, are assumed to be arbitrary positive definite (PD) matrices and $\pi_{x_t}$ is the Jacobian of the optimal feedback policy. This setting is broadly applicable: arbitrary costs are routinely approximated as quadratic along nominal trajectories \citep{li2004iterative, tassa2012synthesis, todorov2005generalized, zhang2025whole}.

The Hessian of the optimal value function $V_{xx_{t+1}}^\pi$ is often viewed merely as a tool for calculation of $\pi_{x_t}$. In the next section we show that it encodes both open-loop and closed-loop information production rates. 

\subsection{Value Function Hessian Decomposition into Information Production Rates}

We show that the backward recursion for the Hessian of the optimal value function admits a decomposition exposing both open-loop and closed-loop information production rates. Each rate is recovered from an auxiliary recursion obtained by removing particular terms from Equation \eqref{eq:dare}.

\begin{definition}[Auxiliary Recursions From the Value Hessian]\label{def:aux_recursions}

Consider the Hessian $V_{xx_t}^\pi$ which satisfies Equation \eqref{eq:dare}. By selectively removing control-dependent terms, two auxiliary backward recursions can be defined:
\begin{align}
    X_t &\coloneq c_{xx_t} + (f_{x_t} + f_{u_t}\pi_{x_t})^\top X_{t+1}(f_{x_t} + f_{u_t}\pi_{x_t}) \label{eq:closed_loop_recursion} \\
    Y_t &\coloneq c_{xx_t} + f_{x_t}^\top Y_{t+1} f_{x_t} \label{eq:open_loop_recursion}
\end{align}
with shared terminal condition $X_T = Y_T = c_{xx_T}$. These recursions are derived in Appendix \ref{app:aux_recursions}.

\end{definition}

\begin{lemma}[Value Hessian Decomposition]\label{lemma:value_decomp}

The asymptotic growth rates of the recursions satisfy:
\begin{align}
    \lim_{T\to\infty} \frac{1}{2T}\ln\det X_0 &= \Lambda^+(f^\mathbf{cl}) = h_{\mbox{ks}}(f^{\mathbf{cl}}) \label{eq:closed_loop_entropy} \\
    \lim_{T\to\infty} \frac{1}{2T}\ln\det Y_0 &= \Lambda^+(f^\mathbf{ol}) \label{eq:open_loop_entropy}
\end{align}
where $\Lambda^+$ denotes the positive LE sum, $h_{\mbox{ks}}$ denotes KSE, and the second equality in Equation~\eqref{eq:closed_loop_entropy} holds under the conditions of Pesin's theorem \citep{pesin1977characteristic} in Equation \eqref{eq:sum_of_positive}.

\begin{proof}[Proof sketch]
Unrolling each backward recursion expresses $X_0$ and $Y_0$ as cumulative sensitivity sums of the form $\sum_{t=0}^{T}\Phi_t^\top W_t \Phi_t$, where $W_t$ is a uniformly bounded positive-definite weighting determined by $c_{xx_t}$ and $\Phi_t$ is a product of Jacobians along the nominal trajectory: closed-loop Jacobians $f_{x_t} + f_{u_t}\pi_{x_t}$ in the case of $X_0$, and open-loop Jacobians $f_{x_t}$ in the case of $Y_0$. The time-averaged $\frac{1}{2}\ln\det$ of any such sum equals the sum of positive Lyapunov exponents of the underlying dynamical system (Lemma~\ref{lemma:kse_equivalence}), independent of the choice of $W_t$ (Lemma~\ref{lemma:weighting_invariance}). Pesin's theorem then identifies the closed-loop sum with KSE. See Appendix~\ref{app:value_decomp}.
\end{proof}

\end{lemma}

Lemma \ref{lemma:value_decomp} shows that open-loop and closed-loop information production rate is contained within Equation~\eqref{eq:dare}. By removing particular terms we can recover one form or the other. We can also show that the $\ln\det$ of the Hessian itself grows at the closed-loop rate in Corollary \ref{corollary:closed_loop_value_hessian}.

\begin{resultbox}
\begin{corollary}[Closed-Loop Information Production-Value Hessian Equivalence]\label{corollary:closed_loop_value_hessian}
Under the same assumptions as Lemma \ref{lemma:value_decomp}, the Hessian of the value function satisfies:
\begin{equation*}
\lim_{T\to\infty} \frac{1}{2T} \ln \det V_{xx_0}^\pi = \Lambda^+(f^\mathbf{cl}) = h_{\mbox{ks}}(f^\mathbf{cl}).
\end{equation*}
\end{corollary}
\end{resultbox}

\begin{proof}[Proof sketch]
The DARE solution admits a cumulative sensitivity representation of the same form as in Lemma~\ref{lemma:value_decomp}, with closed-loop Jacobians $f_{x_t} + f_{u_t}\pi_{x_t}$ and uniformly bounded weighting $W_t = c_{xx_t} + \pi_{x_t}^\top c_{uu_t} \pi_{x_t}$. See Appendix~\ref{app:closed_loop_value_hessian}.
\end{proof}

Notably, Lemma~\ref{lemma:value_decomp} and Corollary~\ref{corollary:closed_loop_value_hessian} hold for any PD and bounded $c_{xx_t}$ and $c_{uu_t}$ which we prove in Lemma~\ref{lemma:weighting_invariance}. CIP therefore depends only on the dynamics, not on the choice of cost weighting.

\subsection{Definition of Controllable Information Production}

We can now formally define CIP from the open-loop information production rate identified in Equation~\eqref{eq:open_loop_entropy}, and establish that it provides an upper bound on the closed-loop information production.

\begin{resultbox}
\begin{definition}[Controllable Information Production]\label{def:cip}
CIP is defined as the open-loop information production along the nominal trajectory:
\begin{equation*}
\mathrm{CIP} \coloneq \Lambda^+(f^\mathbf{ol}).
\end{equation*}
\end{definition}

\begin{theorem}[CIP Upper Bounds Closed-Loop Information Production]\label{theorem:controllable_information_production}
Under optimal perturbation regulation, CIP provides an upper bound on closed-loop information production:
\begin{equation*}
\mathrm{CIP} = \Lambda^+(f^\mathbf{ol}) \geq \Lambda^+(f^\mathbf{cl}) = h_{\mbox{ks}}(f^\mathbf{cl}).
\end{equation*}
\end{theorem}
\end{resultbox}

\begin{proof}
From Definition \ref{def:aux_recursions} the matrix inequality: $X_0 \preccurlyeq V_{xx_0}^\pi \preccurlyeq Y_0$ holds at every horizon $T$. Since $\ln\det$ is monotone on the Loewner order of PD matrices \citep{boyd2004convex}, this gives $\ln\det X_0 \leq \ln\det V_{xx_0}^\pi \leq \ln\det Y_0$. Dividing by $2T$ and taking $T\to\infty$, Corollary \ref{corollary:closed_loop_value_hessian} and Lemma \ref{lemma:value_decomp} yield:
\begin{equation*}
h_{\mbox{ks}}(f^\mathbf{cl}) = \Lambda^+(f^\mathbf{cl}) = \lim_{T\to\infty} \frac{1}{2T} \ln\det V_{xx_0}^\pi 
\leq \Lambda^+(f^\mathbf{ol})= \mathrm{CIP}. \qedhere
\end{equation*}
\end{proof}

CIP is guaranteed to be non-negative and upper-bounds the closed-loop information production. The bound relies on the stabilization properties of the optimal first-order feedback controller and need not hold for a general policy Jacobian, such as one parameterized by a neural network.

\subsection{Efficient Estimation of Information Production}


In practice, propagating the recursion in Equation \eqref{eq:open_loop_recursion} is numerically unstable, growing exponentially at a rate proportional to $\Lambda^+(f^\mathbf{ol})$. To overcome this, we propagate $Y_t$ in the $\ln\det$ domain as a scalar recursion. By Lemma~\ref{lemma:weighting_invariance}, the asymptotic growth rate of $\ln\det Y_t$ is independent of the choice of $c_{xx_t}$, so we set $c_{xx_t} = \mathbf{I}$ without loss of generality.

\begin{resultbox}
\begin{lemma}[Log-domain recursion]\label{lemma:log_domain}
Let $Y_t$ satisfy the recursion in Equation \eqref{eq:open_loop_recursion} with $c_{xx_t} = \mathbf{I}\;$ for all $t$. Then $\ln\det Y_t$ and $Y_t^{-1}$ satisfy the joint recursion:
\begin{equation}\label{eq:log_domain}
    \begin{aligned}
        \ln\det Y_t &= \ln\det Y_{t+1} + \ln\det\!\left(Y_{t+1}^{-1} + f_{x_t} f_{x_t}^\top\right), 
        &\qquad \ln\det Y_T &= 0, \\[4pt]
        Y_{t}^{-1} &= \mathbf{I} - f_{x_t}^\top \!\left(Y_{t+1}^{-1} + f_{x_t} f_{x_t}^\top\right)^{-1} f_{x_t}, 
        &\qquad Y_T^{-1} &= \mathbf{I}.
    \end{aligned}
\end{equation}
\end{lemma}
\end{resultbox}
\begin{proof}
With $c_{xx_t} = \mathbf{I}$, Equation~\eqref{eq:open_loop_recursion} reduces to $Y_t = \mathbf{I} + f_{x_t}^\top Y_{t+1} f_{x_t}$ with $Y_T = \mathbf{I}$. Taking $\ln\det$ and applying Sylvester's identity yields the first equation; inverting and applying the Woodbury identity yields the second.
\end{proof}
\begin{figure}[!ht]
    \centering
    \includegraphics[width=\columnwidth]{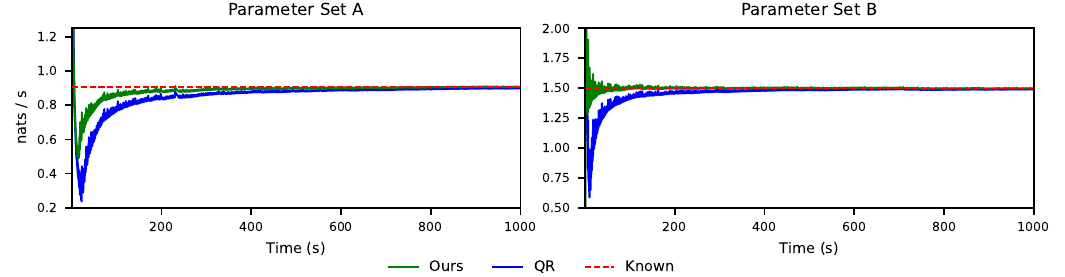}
    \caption{KSE estimation on the Lorenz system in two standard chaotic parameter regimes. \textbf{(A)} $\sigma=10,\; \rho=28,\;b=8/3$, with reference value $h_{\mbox{ks}}(g) \approx 0.906$ \citep{sprott_lorenz}. \textbf{(B)} $\sigma=16,\; \rho=45.92,\; b=4$, with reference value $h_{\mbox{ks}}(g) \approx 1.498$ \citep{wolf1985determining}. Results are shown for the continuous-time analogue of the recursion in Equation~\eqref{eq:log_domain}.
    }
    \label{fig:lorenz_kse}
\end{figure}
The exponential growth of $Y_t$ underflows safely to zero in $Y_t^{-1}$, yielding a numerically stable estimator of the sum of positive Lyapunov exponents and bypassing the QR decomposition required by standard Lyapunov spectrum methods \citep{wolf1985determining, sandri1996numerical, skokos2009lyapunov}. When applied to an autonomous dynamical system $g: \mathcal{X}\to\mathcal{X}$, with no actions, our estimator in Equation~\eqref{eq:log_domain} calculates $\Lambda^+(g)$ which equals $h_{\mbox{ks}}(g)$ under Pesin's conditions in Equation~\eqref{eq:sum_of_positive}. Our estimator converges to the known value of KSE faster than the baseline QR decomposition method in Figure~\ref{fig:lorenz_kse} which demonstrates convergence in two standard chaotic parameter regimes. Appendix~\ref{app:lorenz96} extends this to Lorenz-96~\citep{lorenz1996predictability} up to a state dimension of $100$. 

\section{Experiments}

The goal of our experiments is to evaluate whether IM objectives can drive the emergence of useful physical behaviors, rather than proxy objectives such as state coverage or prediction accuracy. We focus on behaviors such as self-righting, a prerequisite for effective locomotion and energy-efficient movement. Importantly, this notion of usefulness is orthogonal to prior IM objectives (such as diversity or curiosity), which target different properties and are not designed to produce task-relevant behavior directly. Our results show that CIP uniquely induces behaviors that are both physically meaningful and directly beneficial for control. In contrast, existing mutual-information-based methods optimize surrogate objectives (e.g., diversity, prediction error), which do not directly translate into useful physical skills. 

To isolate the role of CIP as an IM, we deliberately adopt a simple controller architecture. We use a random-sampling-based model predictive control (MPC) scheme.
\begin{algorithm}
  \caption{MPC-based CIP Agent}
  \label{alg:cip_mpc}
  \begin{algorithmic}[1]
    \REQUIRE Initial state $x$, horizon $T$, control penalty $\eta$
    \REPEAT
        \STATE \textbf{Solve}: $\mathbf{u}_{0:T-1}^* = \arg\max_{\mathbf{u}_{0:T-1}} \; \mathcal{C}(x, \mathbf{u}_{0:T-1}) - \frac{\eta}{T}\|\mathbf{u}_{0:T-1}\|^2$ 
        \STATE \textbf{Execute}: $x \gets f(x, u_0^*)$
    \UNTIL{convergence at a CIP maximum}
  \end{algorithmic}
\end{algorithm}
Specifically, CIP is computed over a finite horizon $T$ using Equation~\eqref{eq:log_domain} which approximates the asymptotic information production rate. This finite horizon quantity is a control-conditioned tangent space expansion rate which converges to $\Lambda^+(f^\mathbf{ol})$ as $T\to\infty$. We denote this estimate by: $\mathcal{C}(x_0, \mathbf{u}_{0:T-1}) \approx \textrm{CIP}$, where $x_0$ is the initial state at the start of the planning horizon and $\mathbf{u}_{0:T-1} := \{u_0, \ldots, u_{T-1}\}$ is a candidate control sequence. The overall MPC procedure is summarized in Algorithm \ref{alg:cip_mpc} which repeatedly solves an optimization problem with decision variable $\mathbf{u}_{0:T-1}$ and applies the first action in the environment. We use the Improved Cross-Entropy Method (iCEM) \cite{pinneri2020sampleefficientcrossentropymethodrealtime}, a stochastic optimization algorithm commonly used for MPC in continuous control tasks.


We evaluate on four benchmark systems (in Mujoco-MJX \citep{todorov2012mujoco}) that span a range of dynamical complexity:
\begin{itemize}
    \item \textbf{Cart Pole:} A single-link pendulum mounted on a horizontally sliding cart, underactuated with control applied only to the cart. The indirect actuation introduces delayed effects and serves as the simplest testbed in our suite.
    \item \textbf{Double Pendulum:} A two-link serial pendulum, underactuated with torque applied only at the base joint. Its chaotic dynamics test whether CIP can drive both swing-up and fine-grained stabilization through a single passive joint.
    \item \textbf{Triple Pendulum:} A three-link serial pendulum, fully actuated with a motor at each joint. Its strongly chaotic dynamics make it a stringent test of CIP in a higher-dimensional fully actuated regime.
    \item \textbf{Gibbon:} A planar humanoid suspended from fixed anchor points at the feet. This system tests CIP's scalability to high-dimensional underactuated dynamics; full self-righting from a hanging configuration has not, to our knowledge, been demonstrated previously with intrinsic motivation alone.
\end{itemize}

\begin{figure}
    \centering
    \includegraphics[width=\linewidth]{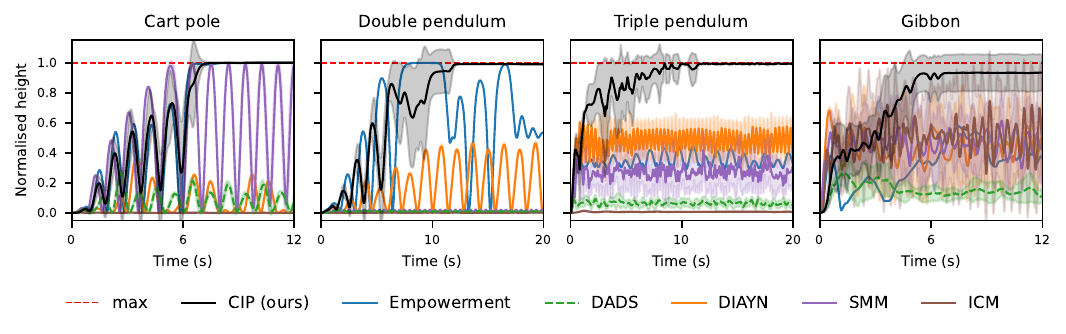}
    \caption{Mean height $\pm$ std of the extremity over time in each environment. For the pendulum environments the extremity is defined as the tip, for the gibbon the height of the head is considered as the extremity. The heights are normalized between $0.0$ and $1.0$ where the latter indicates the fully upright position. All runs start from the hanging position and are averaged over $10$ random seeds, except empowerment which runs deterministically. The skill-based IM (DADS, DIAYN, SMM) are evaluated using the skill with the highest average height. We use the default parameters for the baselines as suggested in their official implementations \citep{laskin2021urlb, sharma2019dynamics}.}
    \label{fig:combined_height}
\end{figure}

\setlength{\intextsep}{4pt}
\begin{wrapfigure}{r}{2.8in}
  \centering
  \vspace{-20pt}
  \includegraphics[width=2.8in]{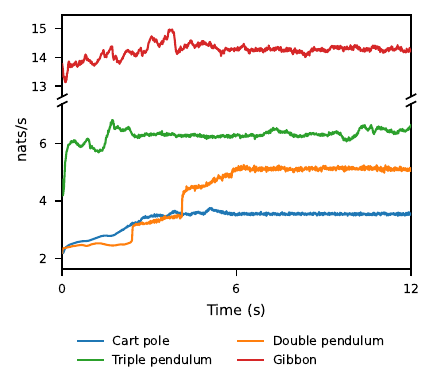}
  \caption{CIP (nats/s) over time during MPC runs.}
  \label{fig:rate}
\end{wrapfigure}

The DIAYN, SMM, and ICM methods are all evaluated using the open source Unsupervised Reinforcement Learning Benchmark (URLB) \citep{laskin2021urlb} implementation. The DADS algorithm is not available in URLB, so we use the official implementation by the authors \citep{sharma2019dynamics}. All algorithms are run using their default hyperparameter choices. To ensure a fair comparison between our MPC-based method and the learning-based methods they are all deterministically reset to the fully downright position during training. 
The empowerment baseline is computed exactly under a linear-Gaussian channel approximation \citep{tiomkin2024intrinsic}, using the dynamics Jacobians along a forward rollout and executing the empowerment-maximizing action; it therefore has the same dynamics access as CIP, although it does not use the iCEM planner. We also implemented the particle-based state-entropy objective of APT \citep{liu2021behavior} inside our iCEM planner where each candidate control sequence is scored by the particle entropy of its own rollout states. In order to evaluate the skill-based IM (DIAYN, DADS, and SMM) the skill with the highest average height is selected. All evaluation runs are conducted using that maximally performing skill. All experiments were conducted on a server with 2 H100 GPUs and 1.0 TB RAM.


\begin{table}
  \centering
  {\small
  \begin{tabular}{l>{\columncolor{gray!15}}ccccccc}
    \toprule
    & \multicolumn{3}{c}{True dynamics (planning)} & \multicolumn{4}{c}{Learned from interaction (RL)} \\
    \cmidrule(lr){2-4} \cmidrule(lr){5-8}
    System              & CIP (ours) & Empowerment & APT (MPC) & DIAYN & DADS & SMM & ICM \\
    \midrule
    Cart Pole         & $\mathbf{0.9996}$ & $\mathbf{0.9996}$   & $0.5488$ & $0.0221$ & $0.0684$ & $0.6171$ & $0.0005$ \\
    Double Pendulum   & $\mathbf{0.9913}$ & $0.5433$            & $0.4503$ & $0.2339$ & $0.0107$ & $0.0112$ & $0.0016$ \\
    Triple Pendulum   & $\mathbf{0.9931}$ & $0.3430$            & $0.4622$ & $0.4865$ & $0.0608$ & $0.2513$ & $0.0069$ \\
    Gibbon            & $\mathbf{0.9319}$ & $0.3988$            & $0.2394$ & $0.4649$ & $0.1430$ & $0.4458$ & $0.5122$ \\
    \bottomrule
  \end{tabular}
  }
  \vspace{1em}
  \caption{Mean extremity height over the final 10\% of the episode, averaged across 10 random seeds. Values are min-max normalized so that $0.0$ corresponds to the initial hanging configuration and $1.0$ corresponds to the fully upright pose.} \label{tab:results}
  \vspace{-2em}
\end{table}

Across all environments only CIP consistently reaches the upright, salient, position as shown in Figure~\ref{fig:combined_height}. Without an extrinsic reward, CIP is able to quickly increase the height and maintain it in the unstable state, despite its reliance on a simple MPC controller. No other baseline aside from empowerment is able to stabilize in the upright position in any of the systems. Empowerment fully succeeds in self-righting in the cart pole environment and partially succeeds in double pendulum but does not stabilize, falling to a suboptimal salient state. Among the skill-based IM DIAYN appears to perform the best, consistently finding oscillatory behavior.

The results in Figure~\ref{fig:combined_height} are summarized in Table~\ref{tab:results} which shows the average height over the last $10\%$ of the episode across $10$ random seeds. In all environments CIP achieves over $90\%$ of the maximum height of the environment, confirming that being upright is indicative of high open-loop information production, where unregulated perturbations will grow the most.

We demonstrate typical values of CIP during MPC runs on each environment in Figure~\ref{fig:rate}, where each agent starts in the hanging down position and successfully self-rights by the end. The ranges of CIP demonstrate that the information production properties of each environment correspond to increasing complexity of the dynamical system, with the gibbon environment having significantly higher CIP than the others. Across all environments, the CIP value starts at a lower level in the hanging position and rapidly increases to a maximal value, indicating that the agent has stabilized in the upright position. The value range on each task is compact, differing by at most $3$ nats/s between the antipodal positions.

\section{Discussion}

This work introduces Controllable Information Production (CIP), a principled IM for continuous control. Unlike existing IM formulations, which are postulated from heuristic or biological inspiration, CIP emerges directly from OC theory. This grounding  provides efficient means for its computation. CIP quantifies the controllable chaos an agent produces in its environment, and we observe that maximizing it gives rise to behaviors that align closely with prior intuitions about how intrinsically motivated agents should act in dynamical control systems. Using a deliberately simple MPC-based controller, we show that CIP consistently drives agents toward salient states at the \emph{edge of chaos} across systems with different complexity. In an environment without an unstable equilibrium, CIP instead yields wall-avoiding, low-coverage behavior (Appendix~\ref{app:ball-in-box}).

The existing IM baselines failed to consistently self-right on our tasks. For skill-based motivators, one likely explanation lies in their core mechanism: they partition state space among latent skills, with each skill claiming a distinct region. This creates selective pressure for state-space diversity, but the upright configuration introduces a competing pressure of its own, as it is the most difficult state to maintain. With easier options available for achieving skill distinguishability, such as learning to hold several stable poses, it is unlikely that any single skill will specialize in reaching and stabilizing the upright state. Curiosity-based methods exhibit a different failure mode, as their rewards depend on an auxiliary predictor whose expressive capacity determines what counts as ``surprising.'' A sufficiently strong predictor learns quickly and ceases to be surprised, while a weaker one may fixate on inherently unpredictable dynamics, such as high-velocity uncontrolled motion \citep{burda2018exploration}.

Our method is promising but has limitations. First, its calculation relies on efficient access to the Jacobians of the dynamics to estimate the information production rate, which constrains the class of simulators where it can be applied to directly \citep{liang2018gpuacceleratedroboticsimulationdistributed, hu2019difftaichi, howell2022dojo}. Second, the scalability of our experiments are limited by the simplicity of the controller. Random-shooting MPC may scale to somewhat higher-dimensional environments than what we demonstrated, but its computational cost grows quickly and will eventually become intractable. MPC controllers are also known to be myopic, as they solve an optimization problem over a finite planning horizon \citep{tassa2012synthesis, pinneri2020sampleefficientcrossentropymethodrealtime}. This limitation means that sometimes the agent can get stuck in a local optimal solution.

In future work, we plan to extend CIP to learning-based controllers and train policies with state-of-the-art RL methods. We also intend to move away from differentiable simulators by adopting learned dynamical models, building on advances in world-model learning from transition data \citep{hansen2023td, hafner2025mastering, hafner2025training}. Finally, while this work emphasizes the connection between CIP and OC theory, demonstrating that KSE can be efficiently computed through the machinery of OC, KSE is natively defined through a refinement of partitions, as outlined in Section \ref{sec:kse}. A complementary direction is to bypass OC entirely and estimate KSE directly from its definition rather than Pesin's theorem, inspired by recent approaches \citep{murphy2024machine, SHIOZAWA2024129531}.

\paragraph{Impact Statement} This paper introduces a novel theoretical framework for IM in continuous environments. The development of theoretically grounded methods for IM has the potential to improve robot learning by reducing reliance on human-designed reward functions and demonstrations, which are often difficult to obtain. Agents driven by IM may exhibit unpredictable behavior due to the lack of explicit human guidance. In real-world robotic systems, such behavior could pose risks to human safety if deployed without appropriate oversight or safeguards. As a primarily theoretical contribution evaluated in simulated environments, the immediate societal impact of this work is limited. Any real-world effects would depend on downstream applications, integration with safety mechanisms, and deployment decisions made by practitioners.

\paragraph{Acknowledgments}
The authors would like to thank Daniel Polani for helpful feedback on this manuscript. Tristan Shah was supported by the Koh Family Scholarship. Stas Tiomkin was supported in part by NSF Award (2513350) and Alliance Innovation Lab in Silicon Valley.

\bibliographystyle{plainnat}
\bibliography{refs}


\appendix

\section{Appendix}

In this section we provide supplementary experimental results, supporting derivations, and full proofs of claims stated in the main text.

\subsection{Ball in Box}\label{app:ball-in-box}

We provide an additional experiment to demonstrate that CIP produces more general behavior beyond the upright stabilization tasks highlighted in the main text. We place a simple agent in a box where it can move freely in the interior. Impacting the walls causes the agent to reverse direction with half of its incoming velocity. A CIP maximizing agent driven by Algorithm~\ref{alg:cip_mpc} avoids the walls and occupies the region in the interior of the box (Figure~\ref{fig:ball-in-box}). The observed behavior follows from the objective: wall contact is dissipative and reduces perturbation growth, so trajectories that contact a wall within the planning horizon have lower information production. An agent that takes random actions uniformly spreads out inside of the box.

\begin{figure}[h]
    \centering
    \includegraphics[width=\linewidth]{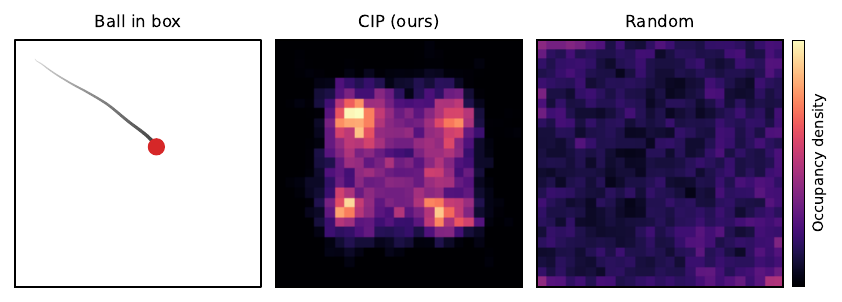}
    \caption{(\textbf{left}) A ``ball in box'' environment where the agent controls the acceleration of a ball inside of a box with dissipative walls. (\textbf{center}) The occupancy heatmap of the ball under a CIP maximizing policy. (\textbf{right}) The occupancy heatmap of the ball under randomized actions.}
    \label{fig:ball-in-box}
\end{figure}

Center-seeking behavior of this kind has been reported for other IM objectives, including causal entropic forces~\citep{wissner_2013} and empowerment~\citep{klyubin2005empowerment}, where the center is favored because it affords more options. CIP recovers the same qualitative behavior from a different quantity.

\subsection{High-Dimensional Validation on Lorenz-96}\label{app:lorenz96}

To evaluate the scalability of our estimator, we apply it to the Lorenz-96 system \citep{lorenz1996predictability}, a standard benchmark for chaotic dynamics whose state dimension $N$ can be chosen freely. We use forcing $F = 8$ and integrate with fourth-order Runge-Kutta at step size $\Delta t = 0.01$ in double precision. Initial conditions are drawn as $F$ plus unit Gaussian noise in each coordinate, followed by $500$ burn-in steps to reach the attractor and a further $500$ steps to settle the estimator state before accumulation begins. As in Figure~\ref{fig:lorenz_kse}, results are computed with a variant of the estimator in Equation~\eqref{eq:log_domain} that shares its asymptotic growth rate, and we compare against QR-based estimation of the sum of positive LE \citep{sandri1996numerical}. Table~\ref{tab:lorenz96} reports the results for $N$ between $10$ and $100$. Our estimator agrees with the QR baseline at every dimension, while tracking the positive sum directly rather than resolving and truncating the full spectrum.

\begin{table}[h]
  \centering
  \vspace{4pt}
  {\small
  \begin{tabular}{lccccc}
    \toprule
    Method & $N = 10$ & $N = 25$ & $N = 50$ & $N = 75$ & $N = 100$ \\
    \midrule
    Ours & $1.98 \pm 0.01$ & $6.27 \pm 0.02$ & $12.80 \pm 0.04$ & $19.26 \pm 0.02$ & $25.70 \pm 0.03$ \\
    QR   & $1.98 \pm 0.01$ & $6.26 \pm 0.02$ & $12.80 \pm 0.04$ & $19.25 \pm 0.02$ & $25.69 \pm 0.03$ \\
    \bottomrule
  \end{tabular}
  }
  \vspace{8pt}
  \caption{$\Lambda^+$ on Lorenz-96 for increasing state dimension $N$, estimated with our recursion and with QR decomposition (mean $\pm$ std over 4 random seeds).}
  \label{tab:lorenz96}
\end{table}

\subsection{Notation}
All vectors are column vectors. Let $\{\bar x_t, \bar u_t\}_t$ denote a nominal trajectory.
Perturbations are defined as $\delta x_t = x_t - \bar x_t$ and $\delta u_t = u_t - \bar u_t$.
We denote Jacobians of the dynamics by
$f_{x_t} := \frac{\partial f(\bar{x}_t, \bar{u}_t)}{\partial \bar{x}_t}$
and
$f_{u_t} := \frac{\partial f(\bar{x}_t, \bar{u}_t)}{\partial \bar{u}_t}$.
The policy Jacobian is
$\pi_{x_t} := \frac{\partial \pi(\bar{x}_t)}{\partial \bar{x}_t}$. The subsequent sensitivity and linearization results are expressed with respect to perturbations originating at time $s$.

\subsection{Preliminaries}\label{app:preliminaries}

\subsubsection{Policy-Aware Linearization}
We assume the existence of a feedback policy $\pi$ that selects future actions between the current state at time $s$ and the future state at time $t$. We choose a linearization which takes into account the total derivative of states in the future. This linearization includes the total perturbation in the final state due to a perturbation in the initial state which propagates through both the dynamics $f$ and the feedback policy $\pi$. Here, the total sensitivity of the state at time $t$ to the state at time $s$ is a product of the closed-loop derivatives:
\begin{equation}\label{eq:closed_loop_prod}
    \frac{d x_t}{d x_s} = \prod_{k = s}^{t-1}(f_{x_k} + f_{u_k} \pi_{x_k})
\end{equation}
where each closed-loop derivative corresponds to the total derivative of the closed-loop dynamics under the policy. If we consider that future controls and states can be expressed with a linear dependence on deviations in the current state and control then we obtain the following linearizations where $\delta x_s = x_s - \bar{x}_s$ and $\delta u_s = u_s - \bar{u}_s$. 
\begin{equation}\label{eq:state_control_linearizations}
    \begin{aligned}
        x_t &\approx \bar{x}_t + \frac{d x_t}{d x_{s+1}}(f_{x_s}\delta x_s + f_{u_s} \delta u_s) \\
        u_t &\approx \bar{u}_t + \frac{d u_t}{d x_{s+1}}(f_{x_s}\delta x_s + f_{u_s} \delta u_s)
    \end{aligned}
\end{equation}
The control linearization involves $\frac{d u_t}{d x_{s+1}} = \pi_{x_t}\frac{d x_t}{d x_{s+1}}$.

\subsubsection{Quadratic Cost Function}

We begin with an quadratic cost function which depends on the current state and control at time $s$:
\begin{equation}\label{eq:quadratic_cost}
    J(x_s, u_s) = \frac{1}{2}\sum_{t=s}^T e_t^\top c_{xx_t} e_t + \frac{1}{2}\sum_{t=s}^{T-1} u_t^\top c_{uu_t} u_t
\end{equation}
where $c_{xx_t}$ and $c_{uu_t}$ are PD cost weighting matrices, Equation \eqref{eq:goal_dist} is the distance between the current state and goal, and Equation \eqref{eq:nominal_goal_dist} is the distance between the nominal state and goal. The cost is calculated over a time horizon from time $s$ to time $T$.
\begin{align}
    e_t &= x_t - g_t \label{eq:goal_dist} \\
    \bar{e}_t &= \bar{x}_t - g_t \label{eq:nominal_goal_dist}
\end{align}

We differentiate the cost in Equation \eqref{eq:quadratic_cost} with respect to the current control $u_s$ and insert the linearizations in Equation \eqref{eq:state_control_linearizations}. In Equation \eqref{eq:gradient_of_cost} the $\bar{e}_t$ term represents the difference between the nominal state $\bar{x}_t$ and the goal $g_t$ as shown in Equation \eqref{eq:nominal_goal_dist}.
\begin{equation}\label{eq:gradient_of_cost}
    \begin{aligned}
        &\frac{\partial J(x_s, u_s)}{\partial u_s} = \underbrace{\sum_{t=s+1}^T e_t^\top c_{xx_t}\frac{d x_t}{d x_{s+1}} f_{u_s} + \sum_{t=s+1}^{T-1} u_t^\top c_{uu_t}\frac{d u_t}{d x_{s+1}}f_{u_s} + u_s^\top c_{uu_s}}_{\textrm{Differentiate cost}} \\
        &\approx \sum_{t=s+1}^T \bar{e}_t^\top c_{xx_t} \frac{d x_t}{d x_{s+1}} f_{u_s} + \delta x_s^\top f_{x_s}^\top\bigg(\frac{d x_t}{d x_{s+1}}\bigg)^\top c_{xx_t} \frac{d x_t}{d x_{s+1}} f_{u_s} + \delta u_s^\top f_{u_s}^\top \bigg(\frac{d x_t}{d x_{s+1}}\bigg)^\top c_{xx_t}\frac{d x_t}{d x_{s+1}} f_{u_s} \\
        &\underbrace{+ \sum_{t=s+1}^{T-1}\bar{u}_t^\top c_{uu_t}\frac{d u_t}{d x_{s+1}} f_{u_s} + \delta x_s^\top f_{x_s}^\top\bigg(\frac{d u_t}{d x_{s+1}}\bigg)^\top c_{uu_t}\frac{d u_t}{d x_{s+1}} f_{u_s} + \delta u_s^\top f_{u_s}^\top \bigg(\frac{d u_t}{d x_{s+1}}\bigg)^\top c_{uu_t}\frac{d u_t}{d x_{s+1}} f_{u_s} + u_s^\top c_{uu_s}}_{\textrm{Insert linearizations}}
    \end{aligned}
\end{equation}

\subsubsection{Gradient of Value}

The gradient of the value function $V_{x_s}^\pi$ acts as a drift term which influences the nominal trajectory to minimize total cost. The gradient of value explicitly depends on both the distance from the goal and the magnitude of control power. An optimal controller will minimize both.
\begin{equation}\label{eq:value_gradient_solution}
    \begin{aligned}
        V_{x_s}^\pi &:= \bigg(\frac{d x_T}{d x_s}\bigg)^\top c_{xx_T} \bar{e}_T + \sum_{t = s}^{T-1}\bigg(\frac{d x_t}{d x_s}\bigg)^\top \bigg[c_{xx_t}\bar{e}_t + \pi_{x_t}^\top c_{uu_t}\bar{u}_t\bigg] \\
    \end{aligned}
\end{equation}
The backwards recursion of the value gradient is given by:
\begin{equation}\label{eq:value_gradient_recursion}
    V_{x_t}^\pi = \bigg(\frac{d x_{t+1}}{d x_t}\bigg)^\top V_{x_{t+1}}^\pi + \bigg[ c_{xx_t} \bar{e}_t + \pi_{x_t}^\top c_{uu_t} \bar{u}_t\bigg]
\end{equation}
with $V_{x_T}^\pi = c_{xx_T}\bar{e}_T$ as the terminal condition. Note the explicit dependency on Equation \eqref{eq:nominal_goal_dist}. 

\subsubsection{Hessian of Value}\label{app:hessian_of_value}

The Hessian of the value function is decomposed into two terms expressed as summations.
\begin{equation}\label{eq:L_and_M}
    L_{s} := \sum_{t=s}^T \bigg(\frac{d x_t}{d x_{s}}\bigg)^\top c_{xx_t}\frac{d x_t}{d x_{s}}, \qquad  M_{s} := \sum_{t=s}^{T-1} \bigg(\frac{d u_t}{d x_{s}}\bigg)^\top c_{uu_t}\frac{d u_t}{d x_{s}}
\end{equation}
The first term $L_s$ measures how a small perturbation $\delta x_s$ in the initial state induces curvature in the cumulative state cost component of Equation \eqref{eq:quadratic_cost}. The second term $M_s$ measures how the initial perturbation induces curvature in the cumulative control cost in Equation \eqref{eq:quadratic_cost}. Both terms in Equation \eqref{eq:L_and_M} can be calculated by backwards recursions in Equation \eqref{eq:L_and_M_recursions}.
\begin{equation}\label{eq:L_and_M_recursions}
    \begin{aligned}
        L_t &= \bigg(\frac{d x_{t+1}}{d x_t}\bigg)^\top L_{t+1} \frac{d x_{t+1}}{d x_t} + c_{xx_t}, \qquad &&L_T = c_{xx_T} \\
         M_t &= \bigg(\frac{d x_{t+1}}{d x_t}\bigg)^\top M_{t+1}\frac{d x_{t+1}}{d x_t} + \pi_{x_t}^\top c_{uu_t} \pi_{x_t}, \qquad &&M_T = \mathbf{0}
    \end{aligned}
\end{equation}

The Hessian of the optimal value function is simply the sum of $L_s$ and $M_s$.
\begin{equation}\label{eq:value_hessian_solution}
    V_{xx_s}^\pi = L_s + M_s
\end{equation}
Its recursion is the sum of the recursions in Equation \eqref{eq:L_and_M_recursions}.
\begin{equation}\label{eq:DARE_closed_loop}
    V_{xx_t}^\pi = \bigg(\frac{d x_{t+1}}{d x_t}\bigg)^\top V_{xx_{t+1}}^\pi \frac{d x_{t+1}}{d x_t} + \pi_{x_t}^\top c_{uu_t}\pi_{x_t} + c_{xx_t}, \qquad V_{xx_T}^\pi = c_{xx_T}
\end{equation}
Notably, for quadratic costs the value Hessian does not depend on the goal $g_t$. This is not a limitation, because in many practical solutions costs are approximated using a second order Taylor expansion \citep{tassa2012synthesis, zhang2025whole}.

\subsection{Proof of the Policy Decomposition Lemma}\label{app:policy_decomp}

We now prove Lemma \ref{lemma:policy_decomp} using the linearizations and value function recursions from Appendix \ref{app:preliminaries}.

\begin{proof}\label{proof:policy_decomp}

Substituting into Equation \eqref{eq:gradient_of_cost} the value gradient (Equation~\eqref{eq:value_gradient_solution}) and value Hessian (Equation~\eqref{eq:value_hessian_solution}) into the transposed cost gradient we get a more simplified expression.
\begin{equation*}
    \bigg(\frac{\partial J}{\partial u_s}\bigg)^\top \approx c_{uu_s} u_s + f_{u_s}^\top V_{x_{s+1}}^\pi + f_{u_s}^\top V_{xx_{s+1}}^\pi f_{x_s}\delta x_s + f_{u_s}^\top V_{xx_{s+1}}^\pi f_{u_s}\delta u_s
\end{equation*}
We can solve for the optimal policy by setting the gradient of the cost equal at zero. The resultant policy is written as:
\begin{equation*}
    \pi(x_s) = \bar{u}_s - (c_{uu_s} + f_{u_s}^\top V_{xx_{s+1}}^\pi f_{u_s})^{-1}(c_{uu_s}\bar{u}_s + f_{u_s}^\top V_{x_{s+1}}^\pi + f_{u_s}^\top V_{xx_{s+1}}^\pi f_{x_s} \delta x_s).
\end{equation*}
Which decomposes into two terms. The extrinsic term $d_s$ is related to the goal due to its dependency on the gradient of value $V_{x_{s+1}}^\pi$ which explicitly depends on the goal as shown in Equation \eqref{eq:value_gradient_recursion}. The intrinsic term $\pi_{x_s}$ only depends on $V_{xx_{s+1}}^\pi$ which has no goal dependency as shown in Equation \eqref{eq:DARE_closed_loop}. We can rewrite the extrinsic feedforward term as $d_s$ and the intrinsic feedback term as $\pi_{x_s}$:
\begin{equation}
    \begin{aligned}\label{eq:intrinsic_extrinisc_split}
        d_s &=   \underbrace{- (c_{uu_s} + f_{u_s}^\top V_{xx_{s+1}}^\pi f_{u_s})^{-1}(c_{uu_s}\bar{u}_s + f_{u_s}^\top V_{x_{s+1}}^\pi)}_{\textrm{Extrinsic Component}} \\
        \pi_{x_s} &= \underbrace{- (c_{uu_s} + f_{u_s}^\top V_{xx_{s+1}}^\pi f_{u_s})^{-1}f_{u_s}^\top V_{xx_{s+1}}^\pi f_{x_s}}_{\textrm{Intrinsic Component}}
    \end{aligned}
\end{equation}
which yields the decomposition stated in Lemma~\ref{lemma:policy_decomp}. 
\end{proof}

\begin{corollary}[Recovery of the Riccati Equation]
Substituting the formula for the policy gradient in Equation \eqref{eq:intrinsic_extrinisc_split} into the backwards recursion for $V_{xx_t}^\pi$ in Equation \eqref{eq:DARE_closed_loop} we obtain DARE.
\begin{equation*}
    V_{xx_t}^\pi = c_{xx_t} + f_{x_t}^\top V_{xx_{t+1}}^\pi f_{x_t} - f_{x_t}^\top V_{xx_{t+1}}^\pi f_{u_t}(c_{uu_t} + f_{u_t}^\top V_{xx_{t+1}}^\pi f_{u_t})^{-1}f_{u_t}^\top V_{xx_{t+1}}^\pi f_{x_t}, \qquad V_{xx_T}^\pi = c_{xx_T}
\end{equation*}
\end{corollary}

\subsection{Justification of Auxiliary Recursions Definition}\label{app:aux_recursions}

Here we provide a derivation of Definition \ref{def:aux_recursions}. First, start with the version of DARE shown in Equation \eqref{eq:DARE_closed_loop} written with closed-loop derivatives and remove (in red) the part associated with control cost. This preserves the closed-loop derivatives while removing irrelevant information about control magnitude.
\begin{equation*}
    V_{xx_t}^\pi = c_{xx_t} + {\color{red}\pi_{x_t}^\top c_{uu_t} \pi_{x_t}} + (f_{x_t} + f_{u_t}\pi_{x_t})^\top V_{xx_{t+1}}^\pi (f_{x_t} + f_{u_t}\pi_{x_t}), \qquad V_{xx_T}^\pi = c_{xx_T}
\end{equation*}
Next, take the standard form DARE in Equation \eqref{eq:dare} and remove (in red) the part associated with control entirely. This effectively measures open-loop perturbation growth along the nominal trajectory.
\begin{equation*}
    V_{xx_t}^\pi = c_{xx_t} + f_{x_t}^\top V_{xx_{t+1}}^\pi f_{x_t} - {\color{red}f_{x_t}^\top V_{xx_{t+1}}^\pi f_{u_t} (c_{uu_t} + f_{u_t}^\top V_{xx_{t+1}}^\pi f_{u_t})^{-1} f_{u_t}^\top V_{xx_{t+1}}^\pi f_{x_t}}, \qquad V_{xx_T}^\pi = c_{xx_T}
\end{equation*}
Now we have two separate backwards recursions:
\begin{align*}
    X_t &\coloneq c_{xx_t} + (f_{x_t} + f_{u_t}\pi_{x_t})^\top X_{t+1}(f_{x_t} + f_{u_t}\pi_{x_t}) \\
    Y_t &\coloneq c_{xx_t} + f_{x_t}^\top Y_{t+1} f_{x_t}
\end{align*}
with a shared terminal condition $X_T = Y_T = c_{xx_T}$.

\subsection{Supporting Lemmas for the Value Hessian Decomposition}\label{app:supporting_lemmas}

\begin{lemma}[Sensitivity-Information Production Equivalence]\label{lemma:kse_equivalence}
Let $\Phi_t = \prod_{k=0}^{t-1} A_k$ be a product of Jacobians evaluated along a nominal trajectory of a dynamical system
$f$ with Lyapunov exponents $\lambda_i$, where $\Lambda^+ = \sum_{\lambda_i>0}\lambda_i$ and define the cumulative sensitivity matrix:
\begin{equation*}
    S_0 = \sum_{t=0}^T \Phi_t^\top \Phi_t.
\end{equation*}
Then:
\begin{equation*}
    \lim_{T\to\infty}\frac{1}{2T}\ln\det S_0 = \Lambda^+.
\end{equation*}
\end{lemma}

\begin{proof}
By the spectral property of the determinant, $\ln \det S_0 = \sum_{i=1}^{n} 
\ln \mu_i(S_0)$, where $\mu_i(S_0)$ are the eigenvalues. By Oseledets' multiplicative 
ergodic theorem \citep{oseledets1968multiplicative}, the state space decomposes into 
invariant subspaces $E_i$ associated with Lyapunov exponents $\lambda_i$:
\begin{equation*}
    \mathbb{R}^{n} = E_1 \oplus \cdots \oplus E_{n}, 
    \qquad
    \left\| \Phi_t v_i \right\| \sim e^{t\lambda_i},
    \quad v_i \in E_i.
\end{equation*}
By the min-max theorem the $i$th eigenvalue of $S_0$ satisfies:
\begin{equation*}
    \lim_{T\to\infty} \frac{1}{T} \ln \mu_{i}(S_0) 
    = \lim_{T\to\infty} \frac{1}{T}\ln \sum_{t=0}^{T} 
    \left\| \Phi_t v_i\right\|^2 
    = \lim_{T\to\infty} \frac{1}{T} \ln \sum_{t=0}^T e^{2t\lambda_i}.
\end{equation*}
Evaluating the geometric series $\sum_{t=0}^T e^{2t\lambda_i} = 
\frac{1 - e^{2(T+1)\lambda_i}}{1 - e^{2\lambda_i}}$ 
for each possible sign of $\lambda_i$:
\begin{equation*}
    \lim_{T\to\infty} \frac{1}{T} \ln \mu_{i}(S_0) = \begin{cases}
        \lim_{T\to\infty} \dfrac{1}{T}\ln\!\bigg(\dfrac{1}{1 - e^{2\lambda_i}}
        \bigg) = 0, \quad & \lambda_i < 0 \\[10pt]
        \lim_{T\to\infty} \dfrac{1}{T} \ln (T + 1) = 0, \quad & \lambda_i = 0 \\[10pt]
        \lim_{T\to\infty} \dfrac{1}{T} \ln\!\bigg(\dfrac{1 - e^{2(T+1)\lambda_i}}
        {1 - e^{2\lambda_i}}\bigg) = 2\lambda_i, 
        \quad & \lambda_i > 0.
    \end{cases}
\end{equation*}
Only the positive Lyapunov exponents survive, so summing over all dimensions yields:
\begin{equation*}
    \lim_{T\to\infty}\frac{1}{2T} \ln \det S_0 = \Lambda^+. \qedhere
\end{equation*}
\end{proof}

\begin{lemma}[Weighting Invariance]\label{lemma:weighting_invariance}
Let $\{W_t\}$ be a sequence of PD matrices with uniformly bounded
eigenvalues,
\begin{equation*}
    0 < \alpha = \inf_t \mu_{\min}(W_t) \leq \beta = \sup_t \mu_{\max}(W_t) < \infty.
\end{equation*}
Let $\Phi_t = \prod_{k=0}^{t-1} A_k$ be a product of Jacobians of a dynamical system 
$f$ with Lyapunov exponents $\lambda_i$, and define: $S_0 = \sum_{t=0}^T \Phi_t^\top 
\Phi_t$. Then
\begin{equation*}
    \lim_{T\to\infty}\frac{1}{2T}\ln\det\bigg(\sum_{t=0}^T 
    \Phi_t^\top W_t \Phi_t \bigg) = \Lambda^+.
\end{equation*}
\end{lemma}

\begin{proof}
The eigenvalue bounds give $\alpha \mathbf{I} \preceq W_t \preceq \beta 
\mathbf{I}$ for all $t$, and therefore
\begin{equation*}
    \alpha S_0 \preceq \sum_{t=0}^T \Phi_t^\top W_t \Phi_t \preceq \beta S_0.
\end{equation*}
Taking the time-averaged log-determinant and applying monotonicity of $\ln\det$,
\begin{equation*}
    \frac{n\ln\alpha}{2T} + \frac{1}{2T}\ln\det S_0 \leq
    \frac{1}{2T}\ln\det\bigg(\sum_{t=0}^T \Phi_t^\top W_t \Phi_t\bigg) \leq
    \frac{n\ln\beta}{2T} + \frac{1}{2T}\ln\det S_0.
\end{equation*}
As $T\to\infty$, the terms $\frac{n\ln\alpha}{2T}$ and $\frac{n\ln\beta}{2T}$ vanish
since $\alpha$ and $\beta$ are constants, and $\frac{1}{2T}\ln\det S_0 \to 
\Lambda^+$ by Lemma \ref{lemma:kse_equivalence} applied to $f$. Both bounds 
therefore converge to $\Lambda^+$, and the result follows by the squeeze theorem.
\end{proof}

\subsection{Proof of the Value Hessian Decomposition Lemma}\label{app:value_decomp}

\begin{proof}
From Equation \eqref{eq:open_loop_recursion}, the solution of the open-loop recursion 
can be written as
\begin{equation*}
    Y_0 = \sum_{t=0}^T \left(\frac{\partial x_t}{\partial x_0}\right)^\top c_{xx_t} 
    \frac{\partial x_t}{\partial x_0}, \qquad
    \frac{\partial x_t}{\partial x_0} = \prod_{k=0}^{t-1} f_{x_k}.
\end{equation*}
This is a weighted sensitivity summation with $W_t = c_{xx_t}$. Since cost Hessians 
are uniformly bounded and PD by assumption, the result:
\begin{equation*}
    \lim_{T\to\infty}\frac{1}{2T}\ln\det Y_0 = \Lambda^+(f^\mathbf{ol}) 
\end{equation*}
follows from Lemma \ref{lemma:weighting_invariance}.

The same argument applies to the closed-loop recursion in Equation 
\eqref{eq:closed_loop_recursion}, which unrolls as
\begin{equation*}
    X_0 = \sum_{t=0}^T \left(\frac{d x_t}{d x_0}\right)^\top c_{xx_t} 
    \frac{d x_t}{d x_0}, \qquad
    \frac{d x_t}{d x_0} = \prod_{k=0}^{t-1} (f_{x_k} + f_{u_k}\pi_{x_k}),
\end{equation*}
with the same weights $W_t = c_{xx_t}$ but with closed-loop Jacobians instead. Lemma \ref{lemma:weighting_invariance} then gives:
\begin{equation*}
    \lim_{T\to\infty}\frac{1}{2T}\ln\det X_0 = \Lambda^+(f^\mathbf{cl}) = h_{\mbox{ks}}(f^\mathbf{cl}). \qedhere
\end{equation*}
The second identification of $\Lambda^+(f^\mathbf{cl})$ with $h_{\mbox{ks}}(f^\mathbf{cl})$ occurs through Equation~\eqref{eq:sum_of_positive}.
\end{proof}

\subsection{Proof of Closed-Loop Information Production-Value Hessian Equivalence Corollary}\label{app:closed_loop_value_hessian}

\begin{proof}
The solution to DARE can be written by adding the solutions to $L_s$ and $M_s$ in 
Equation \eqref{eq:L_and_M}:
\begin{equation*}
    V_{xx_0}^\pi = \sum_{t=0}^{T}\bigg(\frac{d x_t}{d x_0}\bigg)^\top W_t
    \frac{d x_t}{d x_0}, \qquad
    W_t = c_{xx_t} + \pi_{x_t}^\top c_{uu_t} \pi_{x_t}, \quad W_T = c_{xx_T},
\end{equation*}
where the Jacobians $\frac{d x_t}{d x_0} = \prod_{k=0}^{t-1}(f_{x_k} + f_{u_k}
\pi_{x_k})$ are closed-loop. The weights $W_t$ are uniformly bounded and positive 
definite, so the result follows directly from Lemma \ref{lemma:weighting_invariance} applied to the closed-loop system:
\begin{equation*}
    \lim_{T\to\infty}\frac{1}{2T}\ln\det V_{xx_0}^\pi = \Lambda^+(f^\mathbf{cl}) = h_{\mbox{ks}}(f^\mathbf{cl}).
    \qedhere
\end{equation*}
The second identification of $\Lambda^+(f^\mathbf{cl})$ with $h_{\mbox{ks}}(f^\mathbf{cl})$ occurs through Equation~\eqref{eq:sum_of_positive}.
\end{proof}





\subsection{Hyperparameters}\label{app:hyperparameters}

\begin{table}[H]
    \caption{Hyperparameters of the iCEM planner~\citep{pinneri2020sampleefficientcrossentropymethodrealtime} used to optimize our CIP objective across all tasks. Horizon, shots, iterations, and elite fraction follow the notation of \citet{pinneri2020sampleefficientcrossentropymethodrealtime}; $\rho$ is the autocorrelation coefficient of the AR(1) colored-noise process used for action sampling, smoothing refers to the momentum coefficient on the mean of the sampling distribution, and $\eta$ is the control cost penalty defined in Algorithm~\ref{alg:cip_mpc}.}
  \label{tab:hyperparameters}
  \centering
  \begin{tabular}{lccccccc}
    \toprule
    Task              & $\eta$ & Horizon & Shots & Iterations & Elite Fraction & Smoothing & $\rho$ \\
    \midrule
    Cart Pole         & $0.0$ & $400$ & $512$ & $1$ & $0.1$ & $0.1$ & $0.9$ \\
    Double Pendulum   & $0.0$ & $512$ & $1024$ & $10$ & $0.1$ & $0.1$ & $0.9$ \\
    Triple Pendulum   & $2.5$ & $128$ & $2048$ & $10$ & $0.1$ & $0.1$ & $0.9$ \\
    Gibbon            & $9.0$ & $512$ & $1024$ & $1$ & $0.2$ & $0.1$ & $0.9$ \\
    \bottomrule
  \end{tabular}
\end{table}


\end{document}